\documentclass[review]{elsarticle}

\usepackage{lineno,hyperref}
\usepackage{amsthm}
\usepackage{booktabs}
\usepackage{multirow}
\usepackage{subfigure}
\modulolinenumbers[5]

\journal{Journal of \LaTeX\ Templates}

\begin{document}

\begin{frontmatter}

\title{Topological Regularization for Graph Neural Networks Augmentation}

\author[address1]{Rui Song}
\ead{songrui20@mails.jlu.edu.cn}

\author[address1,address4]{Fausto Giunchiglia}
\ead{fausto@disi.unitn.it}

\author[address2]{Ke Zhao}
\ead{zhaoke19@mails.jlu.edu.cn}

\author[address3]{Hao Xu}
\cortext[mycorrespondingauthor]{Corresponding author}
\ead{xuhao@jlu.edu.cn}

\address[address1]{School of Artificial Intelligence, Jilin University, Changchun 130012, China}
\address[address2]{College of Software, Jilin University, Changchun 130012, China}
\address[address3]{College of Computer Science and Technology, Jilin University, Changchun 130012, China}
\address[address4]{Department of Information Engineering and Computer Science,University of Trento, Italy}

\begin{abstract}
The complexity and non-Euclidean structure of graph data hinder the development of data augmentation methods similar to those in computer vision. In this paper, we propose a feature augmentation method for graph nodes based on topological regularization, in which topological structure information is introduced into end-to-end model. Specifically, we first obtain topology embedding of nodes through unsupervised representation learning method based on random walk. Then, the topological embedding as additional features and the original node features are input into a dual graph neural network for propagation, and two different high-order neighborhood representations of nodes are obtained. On this basis, we propose a regularization technique to bridge the differences between the two different node representations, eliminate the adverse effects caused by the topological features of graphs directly used, and greatly improve the performance. We have carried out extensive experiments on a large number of datasets to prove the effectiveness of our model.
\end{abstract}


\end{frontmatter}

\section{Introduction}
In recent years, graph neural networks have shown great advantages in processing graph data, and has been widely applied in many downstream fields, such as chemical molecular structure prediction \cite{fout2017protein,do2019graph}, knowledge graph \cite{zhang2020efficient,xu2020dynamically}, computer vision \cite{qi20173d,wang2019dynamic}, natural language processing \cite{yao2019graph,tu2019multi-hop}, and social networks \cite{qiu2018deepinf,li2019encoding}. However, in practical applications, the acquisition of labeled datas requires a large amount of resources and high cost. Therefore, how to improve the performance of the model with limited training data has always been the research focus.

Data augmentation can expand training data without using new labels, which has been widely applied in computer vision \cite{zhao2019data, cubuk2019autoaugment}. Some common methods flip and translate the image to expand the training data in few-shot learning. However, the complexity and non-Euclidean structure of graph data hinders the development of data augmentation methods similar to those in computer vision \cite{tong2021data}. Moreover, due to the permutation invariance of graph neural networks, those simple operations cannot have a substantial impact on the output of the model. Therefore, changing nodes or edges becomes the first choice for graph data augmentation. Most existing methods aim at deleting and modifying edges according to the induction information of the graph itself to achieve the purpose of graph data augmentation \cite{tong2021data,rong2020dropedge,chen2020measuring,wu2021enhancing}. There are also some studies that construct the dual graph of graph similarity matrix by additional topological features and use it for data augmentation \cite{abel2019topological}. Based on the homogeneity hypothesis of graph neural networks, edges augmentation can be understood as deleting edges between nodes of different labels and expanding edges between nodes of the same labels, so as to improve the probability that the central node obtains effective gain in the propagation process. However, for the semi-supervised node classification task, removing or adding nodes will destroy the original graph structure, which is not conducive to the final node classification.

In this paper, we propose a node augmentation model based on node topology features extension. Inspired by some multiple diffusion and information propagation methods, we use unsupervised node representation methods to learn and obtain the graph topological features, such as Deepwalk \cite{perozzi2014deepwalk}. \cite{abel2019topological} has proved that the explicit addition of topology as input to graph neural network can not improve the accuracy when combining with initial features of nodes. Therefore, we propose a topological feature propagation and regularization scheme to avoid the impact of explicit intervention of topological features on model performance. Specifically, we input the topological features of the graph and the original features into the same dual graph neural network to learn the higher-order features, and use the regularization technology to bridge the gap between the output of the two and reduce the smoothness of the same node to enhance the learning ability of the model. In our method, the topological features are not used as the main features of node prediction, but is used to regularize the model output, so it has a positive effect. Our method can be combined with baseline graph neural network models such as GCN \cite{kipf2017semi-supervised}, GAT \cite{velickovic2018graph}, APPNP \cite{klicpera2019predict} and GCNII \cite{chen2020simple}. Because of the topological regularization term, our model avoids the generation of over-smoothing problem, so the promotion effect is particularly obvious on deep network. In addition, we demonstrate that the introduction of topological features can improve the expressive ability of graph neural networks, and that minimization of topological regularization can prevent over-smoothing. Through a large number of experiments, we prove that our model achieves the state of the art results. Our contributions are as follows:
\begin{itemize}
\item We propose a graph node augmentation method based on topology structure injection and a topology regularization technique. Our method can be combined with other baseline graph neural networks to effectively improve the performance and prevent over-smooth of deep networks.
\item We prove theoretically that the usage of topological features can improve the upper limit of the expressive ability of graph neural networks, and the optimization of topological regularization terms can prevent over-smooth.
\item We conduct a large number of experiments on 5 public datasets to prove the effectiveness of the proposed method. 
\end{itemize}

\section{Related Work}

\subsection{Graph Neural Networks}
The concept of GNN was first proposed in \cite{gori2005a,scarselli2004graphical-based}, which goal is to extend existing neural networks to process graph-structured data.  \cite{bruna2014spectral} defines the convolution in the spectral domain with the help of Fourier transform based on Laplace matrix, \cite{kipf2017semi-supervised} simplifies it to GCN with two layers. In recent years, many variants of GNN have been proposed to solve different problems on graph neural networks. GIN is proposed to explore the expressive power of graph neural networks \cite{xu2019how}. APPNP \cite{klicpera2019predict} and GCNII \cite{chen2020simple} are outstanding in solving the over-smoothing problem of deep network. \cite{pei2020geom-gcn, zhu2021graph} improve GNN modeling ability of non-homogeneous graph. There are also some new studies to improve the performance of the semi-supervised node classification GNN by combining the boosting approaches \cite{sun2021adagcn}. However, there are few researches on data augmentation for graph.

\subsection{Graph Augmentation}
Although data augmentation has achieved great success in computer vision \cite{zhao2019data,cubuk2019autoaugment} and natural language processing \cite{xie2020unsupervised}, there are few studies on graph-based data augmentation methods. DropEdge \cite{rong2020dropedge} has achieved good results on multiple baseline models by randomly dropping edges between nodes to slow down over-smooth of the deep networks. ADAEDGE \cite{chen2020measuring} enhances the graph by iteratively predicting high confidence nodes with the same or different labels, adding or removing edges between them. GAUG \cite{tong2021data} adaptively learns the EDGE probability predictor from the graph encoder, adding high-probability edges to the graph and removing low-probability edges. EGNN \cite{wu2021enhancing} enhances GNN by introducing the auxiliary task, node pair classification or link prediction. There are also ways to augment the node features with external features. \cite{abel2019topological} adds an additional adjacency matrix between distant nodes that are topological similar to improve its accuracy. 

\section{Proposed Method}
In this section, we give some notation definitions and elaborate on our proposed method.

\subsection{Notation}
Given a undirected graph $G=\{V,E\}$ consisting of a set of nodes $V$ and a set of edges $E$, the aim of node classification task is to learn a function $\chi: V_l \to V_u$, where $V_l$ represents labeld nodes and $V_u$ represents unlabeled nodes. $A \in \mathcal{R}^{n\times n}$ is the adjacency matrix of graph $G$, and $X \in \mathcal{R}^{n\times d}$ is node feature matrix, where $n$ is the number of nodes and $d$ is the dimensions of the input feature. $D \in \mathcal{R}^{n\times n}$ represents the diagonal degree matrix given by $D={\{d_1,...d_n\}}$, where $d_i$ is the degree of node $i$. Then, the adjacency matrix after trick normalization is $\hat{A}=\tilde{D}^{-1/2}\tilde{A}\tilde{D}^{-1/2}$, where $\tilde{A}$ is the adjacency matrix with self-loop and $\tilde{D}$ is the corresponding diagonal degree matrix. In general, the propagation process of GNN can be defined as:
\begin{equation}
	H_{l+1}=GNN(\hat{A},H_l,W_l)
\end{equation}
where, $W_l$ is a learnable parameter matrix and $H_l$ is input matrix of current layer. For $l=0$, we have $X=H_0$. 

\subsection{Model}
\begin{figure}[t!]
	\begin{center}
		\includegraphics[width=0.99\textwidth]{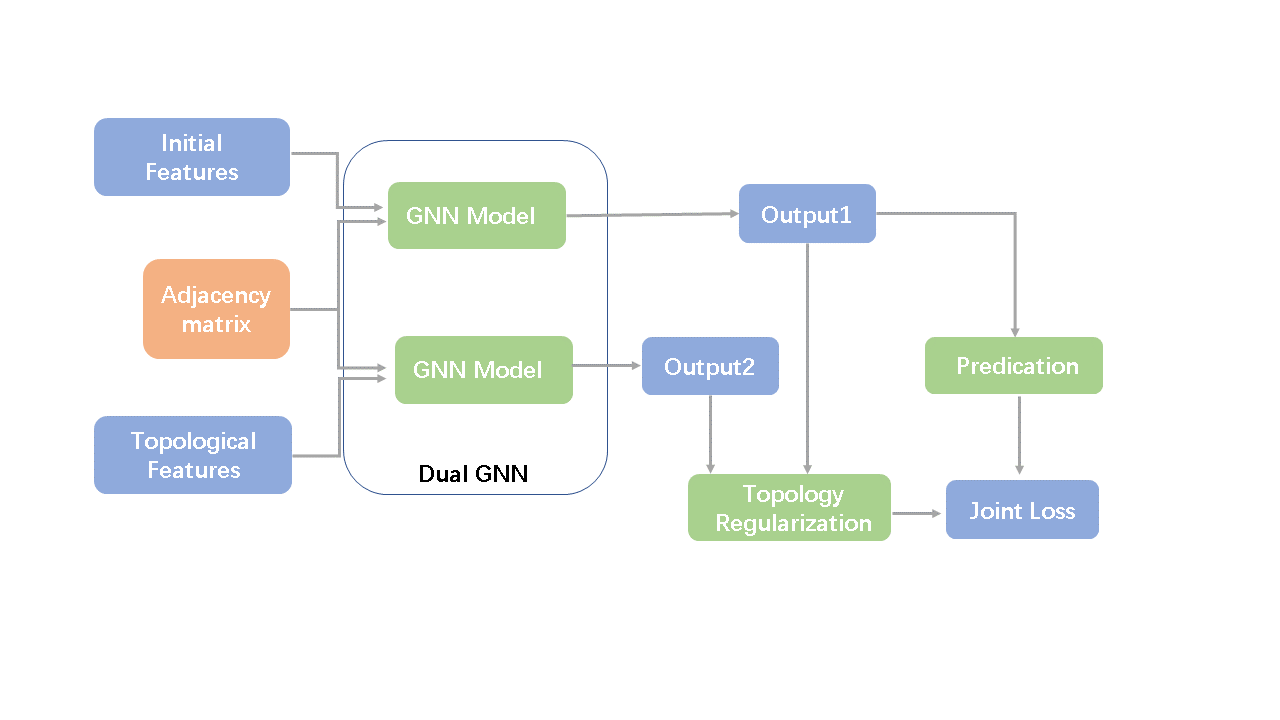}
		\caption{Model Structure.}
		\label{fig:model}
	\end{center}
\end{figure} 
Our key idea is to enhance node features based on topology information. As figure \ref{fig:model} shows, our model is divided into three main parts. Firstly, we obtain the topological features of nodes with the help of Node2Vec \cite{grover2016node2vec}. Then, we carry out equivalent feature propagation with two identical structures called dual graph neural network. Finally, we construct the joint loss function with topological regularization term.

\subsubsection{Topological Features Injection} 
Effective topological structure information can enhance the expressive ability of graph neural network. For some shallow GNN, such as GCN and GAT, the expression ability of the model cannot break through its 2-hop neighborhoods. Some graph walking algorithms can explore the distant neighbors, thus breaking the lack of expression ability brought by this shallow network. \cite{xu2019how} has proven that GNN is at most as powerful as WL tests in differentiating graph structures, which illustrates the inadequacy of GNN expression. But \cite{davide2020are} has confirmed that $k$ length Random Walk with Restart(RWR) algorithm has different performance from $k$ iterations of 1-WL, \cite{jin2020gralsp} shows that Anonymous Random Walks(ARW) is sufficient to reconstruct the local neighborhood within a fixed distance. This provides evidence that the walking algorithms can improve the upper limit of GNN expression ability. However, simply adding the transition probability among nodes as enhancement features will increase the input of size $n\times n$, which is unacceptable for large graphs. Therefore, we use Node2Vec to materialize the probability into dense features to reduce feature space and memory consumption. 

We prove the relationship between the random walk method and 1-WL test.
\newtheorem{thm}{\bf Theorem}
\begin{thm}\label{thm1}
If the k-step 1-WL test proves that the two graphs are non-isomorphic, then the probability transition matrices of their random walk should be different. 
\end{thm} 

\begin{proof}
Given two non-isomorphic graphs, $G_1=\{V_1,E_1\}$ and $G_2=\{V_2,E_2\}$. For the initial graph with all nodes labeled as 1, if two nodes are of the same degree, then their color is the same, because the color of each node is given by multiset of its neighbors' color. And so on until $k-1$ step, WL test could not tell the difference between the two nodes, which means the multiset with the degrees of the neighbours at distance $k$ are the same. So the fact that 1-WL test can distinguish the difference between them at $k$ step means that it cannot distinguish the difference between the two at $k-1$ step, which indicates  the degree distribution of nodes is different in $k$ step. 

Consider two nodes, $u\in V_1,v\in V_2$, if they have different degree distributions at $k$-step, the walks of length $k$ from both of them will have different choices on the $k$-th edge. Therefore, for $G_1$ and $G_2$, their $k$-step random walk produces different probability matrices. 
\end{proof}
For graph representation learning methods based on random walk, different probability transition matrices will generate different potential walk paths, and the embedding of paths will generate different topological feature vectors. This means that we are able to capture features that GNN cannot by means of DeepWalk, Node2Vec, etc. After obtaining the topology features, in order to maintain the stability of the model, the following normalized methods are used to refine the topology features:
\begin{equation}
	v = \frac{exp(v_i)}{\sum_{j} exp(v_j)}
\end{equation}

\subsubsection{Dual Graph Neural Network} 
Due to the degradation of model performance caused by direct usage of topological features, we design a dual graph neural network structure to propagate the graph feature and the original feature respectively. The same network architecture can better maintain the consistency between the two different features. In practice, we can also use different GNNs model for the structure. We will further explore the performance differences between the different GNN models in the section \ref{section:experiment}.

In addition, in order to facilitate the subsequent regularization, we also need to ensure that the original features have the same dimension as the topological features. For APPNP and GCNII, we change the dimensions of the multi-layer perceptron(MLP), and for GCN and GAT, we adjust the hidden size directly. Taking APPNP as an example, the process of dual propagation can be summarized as:
\begin{equation}
	H_{init}^0 = MLP(H_{init})
\end{equation}
\begin{equation}
	H_{topo}^0 = MLP(H_{topo})
\end{equation}
\begin{equation}
	H_{init}^K = APPNP(H_{init}^{K-1})
\end{equation}
\begin{equation}
	H_{topo}^K = APPNP(H_{topo}^{K-1})
\end{equation}
where $H_{init}^0,H_{topo}^0\in \mathcal{R}^{n\times h}$, $h$ denotes the hidden size. The propagation formula of APPNP is as follows:
\begin{equation}
	H^k = ReLU((1-\alpha)\hat AH_{k-1} + \alpha H^0)
\end{equation}

\subsubsection{Topological Regularization} 
In order to make reasonable use of topology information, we propose a topological regularization method, which does not explicitly use topological features, but provides positive guidance for model learning with the topological features. Topological regularization is based on a reasonable assumption that there should be similarities between different feature representations of the same node, and differences between different nodes to prevent over-fitting and over-smoothing. Formally, the topological regularization of node $i$ is defined as:
\begin{equation}
	\mathcal{L}_i = \frac{1}{N}\frac{ \sum_{j\in{\{V-\{i\}\}}}{H_{init[i]}^K \cdot H_{topo[j]}^K} }{||H_{init[i]}^K||_2 ||H_{topo[j]}^K||_2} + (H_{init[i]}^K-H_{topo[i]}^K)^2
\end{equation}
where $H_{init[i]}^K$ denotes the $i$-th row of matrix, ${V-\{i\}}$ represents the difference set of nodes. In particular, we call the term to the left of the $+$ Similar Regularization(SR), and the term to the right of the $+$ Differ Regularization(DR). We will discuss the influence of different regularization terms on the model in section \ref{section:experiment}. The regularization of the graph is written as the average of the regularization values for each node:
\begin{equation}
	\mathcal{L}_{reg} = \frac{\sum_{i\in V}{ \mathcal{L}_i }}{N}
\end{equation}
Finally, the joint loss function is defined as:
\begin{equation}
	\mathcal{L} = \sum_{i=0}^{N}{Y_ilogZ_i} + \lambda\mathcal{L}_{reg}
\end{equation}
where $Z\in \mathcal{R}^{n\times c}$ is calculated by:
\begin{equation}
	Z = softmax(Liner(H_{init}^K,W))
\end{equation}
where $c$ is the classes, $W\in \mathcal{R}^{d\times c}$ is a learnable parameter. The combined optimization of loss function can enhance the robustness and anti-over-smooth ability of the model, which is particularly important in deep graph neural networks. Over-smooth refers to the fact that the nodes in the deep graph network have excessive feature propagation which makes it difficult to distinguish the nodes from different clusters \cite{li2018deeper}. According to the theoretical analysis in \cite{liu2020towards}, we summarize the convergence of networks under two different propagation modes when over-smoothing occurs.

Define $e=[1,...,1]$ as a row vector, whose elements are all 1, function $\Psi(x)=\frac{x}{sum(x)}$ normalizes a vector to sum to 1 and function $\Phi=\frac{x}{||x||}$ normalizes a vector such that its magnitude is 1. For $\hat{A}=\tilde{D}^{-1/2}\tilde{A}\tilde{D}^{-1/2}$ used by GCN and $\hat{A'}=\tilde{D}^{-1}\tilde{A}$ used by GraphSAGE \cite{hamilton2017inductive} corresponding to two different propagation mechanisms, their convergence is respectively given by the following theorems:

\begin{thm}\label{thm2}
	Given a connected graph $G$, ${\lim_{k \to +\infty}{\hat{A}^k}}=\prod$, where each row of $\prod$ is $\Psi(e\tilde{D})$.
\end{thm} 

\begin{thm}\label{thm3}
	Given a connected graph $G$, ${\lim_{k \to +\infty}{\hat{A'}^k}}=\prod'$, where $\prod'=\Phi(\tilde{D}^{1/2}e^T)(\Phi(\tilde{D}^{1/2}e^T))^T$.
\end{thm} 

According to the above convergence theorems, we know that every row of $\prod$ is the same, and every row of $\prod'$ is proportional to the square root of the degree of the node. Therefore, we have the following corollary and theorem:

\newtheorem{corollary}{Corollary}
\begin{corollary}
 	Given a connected graph $G$, the over-smooth is equivalent to that ${\lim_{k \to +\infty}{(H^k_i-d_{ij}H^k_j)}}=0$, where $k$ denotes GNN model layer, $H^k_i$ denotes output features of node $i$ in $k$ layer, and $d_{ij}=1$ for $\hat{A}=\tilde{D}^{-1/2}\tilde{A}\tilde{D}^{-1/2}$, $d_{ij}=\sqrt{\frac{d_i}{d_j}}$ for $\hat{A'}=\tilde{D}^{-1}\tilde{A}$, where $d_i$, $d_j$ denote the degree of node $i$, $j$, respectively. 
\end{corollary}

\newtheorem{thm4}{\bf Theorem}
\begin{thm}\label{thm4}
Minimizing topological regularization terms $\mathcal{L}_{reg}$ prevents over-smooth.
\end{thm} 
\begin{proof}
We use contradiction to prove it. Let's assume that minimizing $\mathcal{L}_{reg}$ and over-smooth can occur simultaneously, that is for any two nodes $i,j\in V$, we have $H_{init[i]}^K-d_{ij}H_{init[j]}^K=0$. For an feature matrix with any element is not $\vec{0}$, minimizing $\mathcal{L}_{reg}$ means keeping $\mathcal{L}_{reg}$ as close to zero as possible, equal to 0 ideally. Thus, we have $(H_{init[i]}^K-H_{topo[i]}^K)^2=0$, which means that any two identical node features have the same representation in dual GNN, that is $H_{init[i]}^K=H_{topo[i]}^K$. At this point, $\mathcal{L}_{reg}$ can be rewritten as:
\begin{equation}
	\mathcal{L}_i = \frac{1}{N}\frac{ \sum_{j\in{\{V-\{i\}\}}}{H_{init[i]}^K \cdot H_{init[j]}^K} }{||H_{init[i]}^K||_2 ||H_{init[j]}^K||_2} = 0
\end{equation}
For $H_{init[i]}^K-d_{ij}H_{init[j]}^K=0$, the $\mathcal{L}_i$ becomes:
\begin{equation}
	\mathcal{L}_i = \frac{1}{N}\frac{ \sum_{j\in{\{V-\{i\}\}}} \frac{1}{d_{ij}} {H_{init[i]}^K}^2 }{||H_{init[i]}^K||_2^2} = 0
\end{equation}
That is for any node in graph, $H_{init[i]}^K=\vec{0}$. So the contradiction arises and the theorem holds.
\end{proof}

\section{Experiment}
\label{section:experiment} 
In this section, we describe the datasets used and the baseline algorithm we compared. Then we analyze the performance of our proposed approach on different datasets. After that, we perform ablation studies to verify the model performance and explore some key parameters.

\begin{table}[t]
	\centering
	\resizebox{\textwidth}{15mm}{
	\begin{tabular}{lccccccc}
		\toprule
		Dataset  & Classes & Nodes & Edges & Features & Train & Validation & Test\\
		\midrule
		Cora     &7	&2,708	&5,429	&1,433 & 20 per class & 500 & 1000\\
		Citeseer  &6 &3,372	&4,732 &3,703 & 20 per class & 500 & 1000\\
		Pubmed   &3	&19,717	&44,338	&500& 20 per class & 500 & 1000\\
		Amazon Photo   &8	& 7,650 & 119,043 & 745 &20 per class & 30 per class & rest nodes\\
		Amazon Computer   &10	& 13,752	& 245,778 & 767 & 20 per class & 30 per class & rest nodes\\
		\bottomrule
	\end{tabular}}
	\caption{Statistic of datasets.}
	\label{tab:datasets}
\end{table}

\subsection{Datasets and Baselines}
\textbf{Datasets}. We used five common datasets to evaluate our methods. For Cora, Citeseer, and PubMed, we followed the split standard mentioned in \cite{kipf2017semi-supervised}. For Amazon Photo and Amazon Computer from \cite{mcauley2015image-based}, we randomly choose 20 nodes per class for training, 30 per class for validation and the rest for test. It is worth noting that there are some isolated nodes in Amazon Photo that can not capture the topology features. But instead of deleting them, we use vectors that are all $1$s as topological features of them. Unless otherwise noted, all data preprocessing is the same as \cite{kipf2017semi-supervised}. The statistics of the datasets are shown in Table \ref{tab:datasets}. 

\textbf{Baselines}. We apply our method on two shallow graph networks GCN \cite{kipf2017semi-supervised}, GAT \cite{velickovic2018graph}, and two deep graph models, APPNP \cite{klicpera2019predict} and GCNII \cite{chen2020simple} respectively. We have different parameter settings for different datasets and models. For GCN, GAT and GCNII, we strictly use the parameters given in the original paper. For APPNP, we adjust them to get better experimental results compared with those in the original paper. More specific parameters are given in the appendix. We also use SR and DR separately for comparison to determine the effects of different regularization methods. The overall results can be seen in table \ref{tab:perform}. 

\subsection{Overall Results}
\begin{table}[t]
	\centering
	\resizebox{\textwidth}{45mm}{
	\begin{tabular}{lccccccc}
		\toprule
		Model& & Cora  & Citeseer & Pubmed & Amazon Photo & Amazon Computer\\
		\midrule
		\multirow{4}{*}{GCN} 
		 & Vanilla & 81.9±0.4 & 70.7±0.4 & 79.0 & 83.39±0.59 & 73.0±0.39\\
		 &SR   & 82.78±0.12 & 71.18±0.32 & 79.2±0.26 & 85.13±0.54 & 73.23±0.22\\
		 &DR   & 82.68±0.10 & 71.3±0.37 & 79.18±0.28 & 84.49±0.52 & 72.23±0.15\\
		 &SR+DR   & \textbf{83.44±0.34} & \textbf{71.76±0.15} & \textbf{80.64±0.12} & \textbf{85.72±0.42} &  \textbf{73.76±0.46}\\
		\bottomrule
		\multirow{4}{*}{GAT} 
		 & Vanilla & 81.02±0.32 & 71.2±0.36 & 78.98±0.19 & 82.99±0.01 & 73.07±0.13\\
		 &SR   & 80.68±0.15 & 71.66±0.05 & 79.06±0.10 & 83.67±0.08 & 75.27±0.15\\
		 &DR   & 82.96±0.48 & 70.86±0.15 & 79.18±0.22 & 83.25±0.31 & 75.13±1.7\\
		 &SR+DR   & \textbf{83.54±0.21} & \textbf{72.36±0.21} & \textbf{79.38±0.12} & \textbf{83.8±0.22} & \textbf{75.92±0.85}\\
		\bottomrule
		\multirow{4}{*}{APPNP} 
		 & Vanilla & 83.2±0.52 & 71.86±0.22 & 80.24±0.14 & 85.54±0.55 & 75.02±0.09\\
		 &SR   & 83.6±0.21 & 72.32±0.46 & 80.22±0.07  & 86.25±0.12 & 75.58±0.23 \\
		 &DR   & 83.44±0.44 & 71.92±0.38 & 80.02±0.10  & \textbf{86.6±0.09} & 75.73±0.16 \\
		 &SR+DR   & \textbf{85.2±0.04} & \textbf{73.82±0.16} & \textbf{83.02±0.15}  & 86.18±0.30 & \textbf{76.03±0.09} \\
		\bottomrule
		\multirow{4}{*}{GCNII} 
		& Vanilla & 83.54±0.31 & 72.12±0.24 & 80.22±0.07 & 86.28±0.07 &  76.34±1.03 \\
		&SR   & 83.32±0.25 & 72.26±0.26 & 80.0±0.2 & 86.15±0.08 & 77.04±0.09\\
		&DR   & 83.58±0.15 & 72.14±0.4 & 80.04±0.1 & 86.15±0.06 & 76.75±0.05\\
		&SR+DR   & \textbf{85.8±0.26} & \textbf{74.92±0.21} & \textbf{81.22±0.19} & \textbf{87.21±0.13} & \textbf{78.58±0.51}\\
		\bottomrule
	\end{tabular}}
	\caption{Test accuracy(\%) for different model. SR+DR means that we use the topological regularization method proposed in this paper. Each model was run 10 times, and the mean ± standard deviation was taken as the experimental result.}
	\label{tab:perform}
\end{table}

For all datasets and all baseline models, the topological regularization items can improve the performance of the model in most cases except APPNP on Amazon Photo, which demonstrates the effectiveness of reasonable topological feature injection. Although SR and DR are sometimes useful, they don't have positive effects on all models and datasets. We note that topological regularization in particular improves the performance of the deep models, because over-smooth occurs significantly at the deep layers. We also noted that the performance of the Vanilla model determines the upper limit of topological regularization, and the better the Vanilla model, the higher the final classification accuracy. This shows that topological regularization is extensible and can be further improved by combining with some latest methods. 

\subsection{Using topological features directly}
In order to verify the necessity of topological regularization, we directly use topological features, and concatenate topological features and original features together as input, and observe the model performance under different datasets as Figure \ref{fig:bar} shows. 

Because topological features are primitive and crude, using topological features directly as input to the model can have negative effects compared to the vanilla models. But the topological features still contain useful information that allows models to make better predictions than random guesses. Thus, it seems intuitive that there might be some gain from directly concatenating the original features with the topological features as input. However, in some cases, simple concatenating is not enough to eliminate the noise contained in the topological features, such as GCNII on Cora and GAT, APPNP, GCNII on Citeseer. But topological regularization can avoid this problem, because topological features don't explicitly participate in prediction, they provide useful guidance for model updating someway. 

\begin{figure}[h]
	\centering
	\subfigure[Cora]{
		\begin{minipage}[t]{0.49\linewidth}
			\centering
			\includegraphics[width=6.6cm]{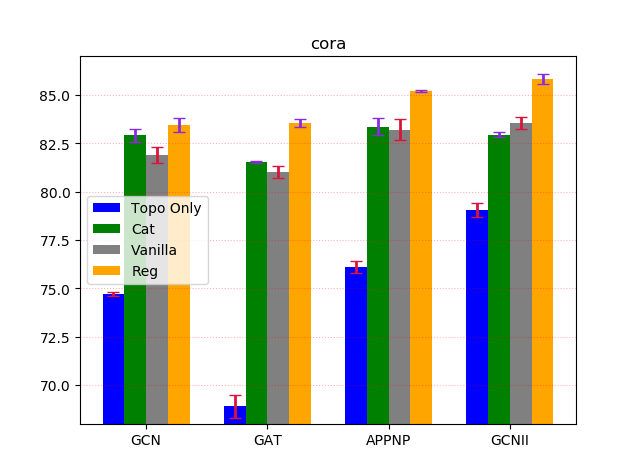}
		\end{minipage}%
	}%
	\subfigure[Citeseer]{
		\begin{minipage}[t]{0.49\linewidth}
			\centering
			\includegraphics[width=6.6cm]{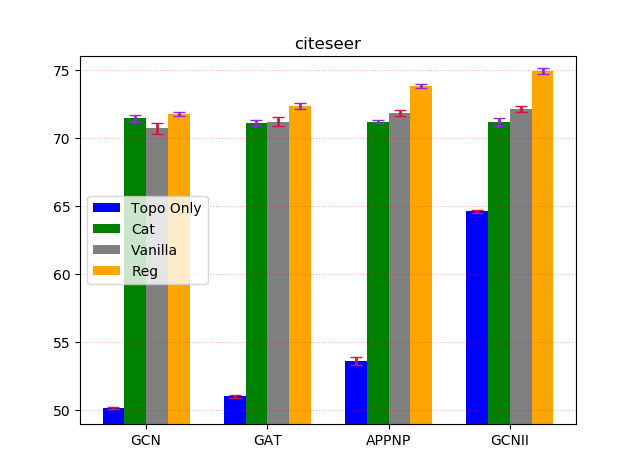}
		\end{minipage}%
	}%
	\caption{The influence of different methods of using topological features on the models for cora and citeseer. 'Topo Only' means that only topological features are used, 'Cat' means that we concatenate topological features and original features together as the input.}
	\label{fig:bar}
\end{figure}

\subsection{Random Noise Injection}
\begin{figure}[h]
	\centering
	\subfigure[APPNP]{
		\begin{minipage}[t]{0.49\linewidth}
			\centering
			\includegraphics[width=6.6cm]{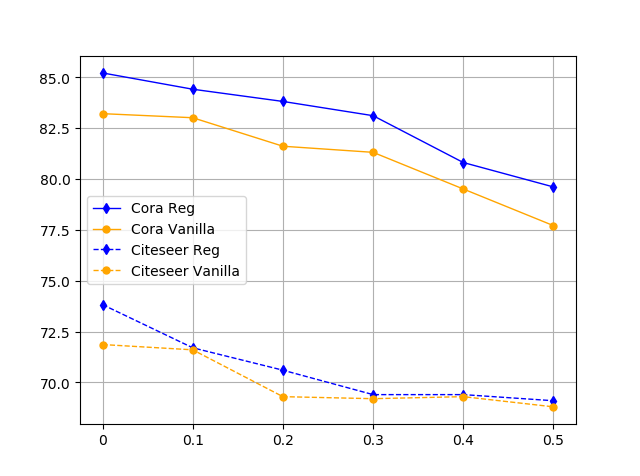}
		\end{minipage}%
	}%
	\subfigure[GCNII]{
		\begin{minipage}[t]{0.49\linewidth}
			\centering
			\includegraphics[width=6.6cm]{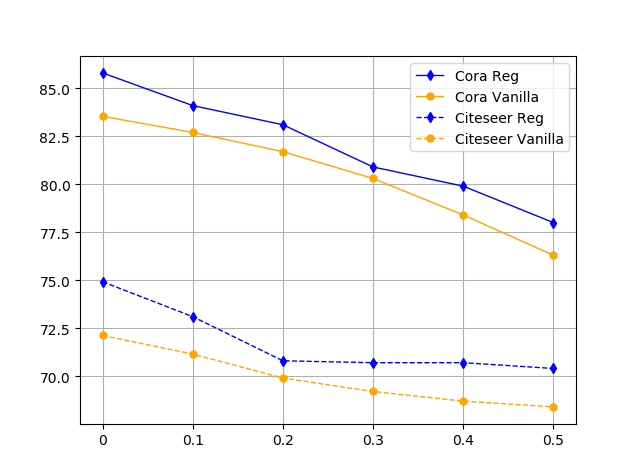}
		\end{minipage}%
	}%
	\caption{Variation of model performance when noise is introduced.}
	\label{fig:masks}
\end{figure}
In order to test the resistance to noise, we artificially introduced noise by randomly masking the initial features of nodes to $\vec{0}$, and observe the changes in the accuracy of the model. As is shown in Figure \ref{fig:masks}, as the proportion of the random masks increases from $0.1$ to $0.5$, the performance of the model gradually declines, but the method using topological regularization terms still has better performance, which shows that the topological regularization can provide positive guidance for the model performance even when it contains a lot of noise. We also notice that excessive noise injection has more negative effect on Citeseer, especially in APPNP, where the topological regularization terms barely worked anymore when the proportion of mask was greater than $0.1$. 

\subsection{Changes of Regularization Loss}
\begin{figure}[h]
	\centering
	\subfigure[GCNII on Cora]{
		\begin{minipage}[t]{0.49\linewidth}
			\centering
			\includegraphics[width=6cm]{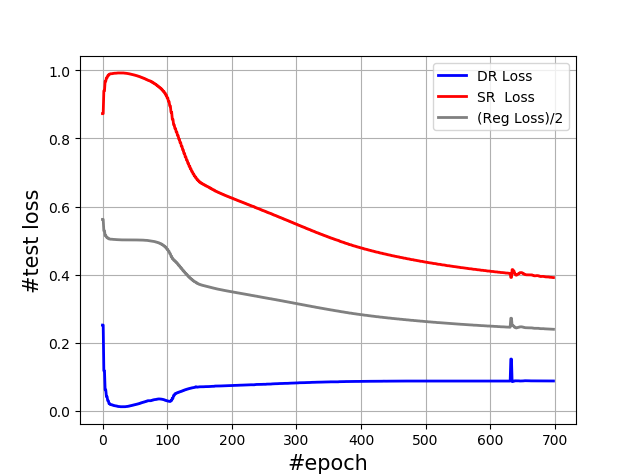}
		\end{minipage}%
	}%
	\subfigure[GCNII on Citeseer]{
		\begin{minipage}[t]{0.49\linewidth}
			\centering
			\includegraphics[width=6cm]{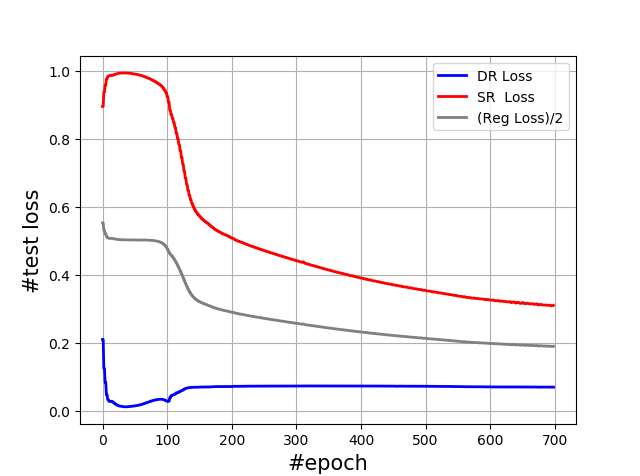}
		\end{minipage}%
	}%
	\quad
	\subfigure[APPNP on Cora]{
		\begin{minipage}[t]{0.49\linewidth}
			\centering
			\includegraphics[width=6cm]{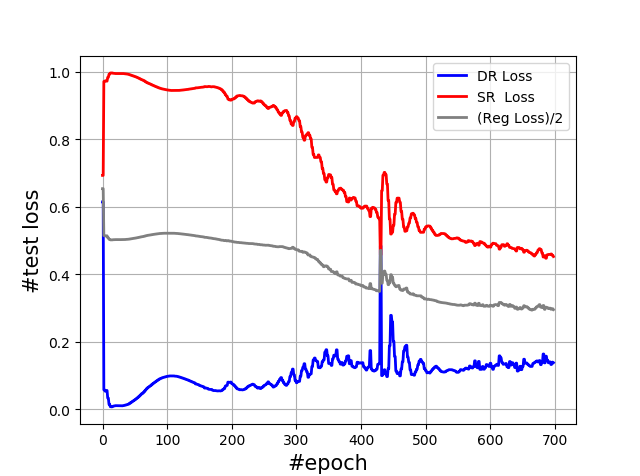}
		\end{minipage}%
	}%
	\subfigure[\textsc{APPNP} on Citeseer]{
		\begin{minipage}[t]{0.49\linewidth}
			\centering
			\includegraphics[width=6cm]{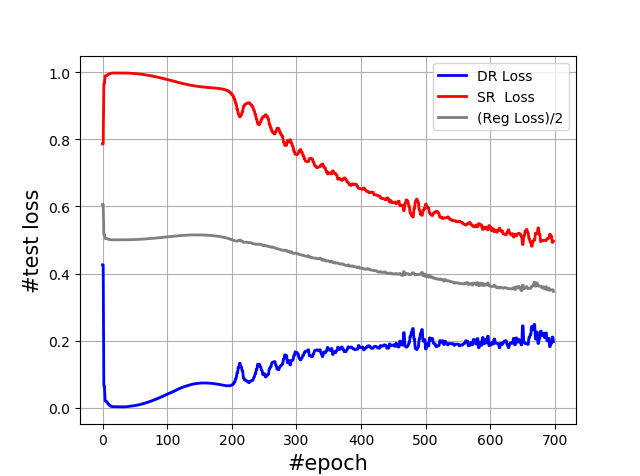}
		\end{minipage}%
	}%
	\caption{GCNII and APPNP regularization loss by epoch on Cora and Citeseer.}
	\label{fig:loss}
\end{figure}

In order to further study the working principle of topological regularization, we output the regularization loss under different datasets as the epoch changes in Figure \ref{fig:loss}. In general, the regularization loss under different models and different datasets has similar variation rules. In the first several epochs, both SR loss and DR loss are relatively large, but with opposite changes: SR loss gradually increases while DR loss decreases. After that, SR decreases rapidly and DR increases gently. But the overall regularization loss keeps decreasing. This indicates that the topological regularization term takes the overall loss as the optimization goal, in which SR plays a dominant role in the optimization process while DR only plays an auxiliary role, so DR only needs to be maintained at a low level instead of falling all the time. Moreover, since the number of node pairs to be considered for optimizing DR is less than that for SR, there is a significant decrease in DR at the beginning, which means that DR is easier to optimize. 

\subsection{Evidence for preventing over-smooth}
\begin{figure}[h]
	\centering
	\subfigure[Without topological regularization.]{
		\begin{minipage}[t]{0.49\linewidth}
			\centering
			\includegraphics[width=6.6cm]{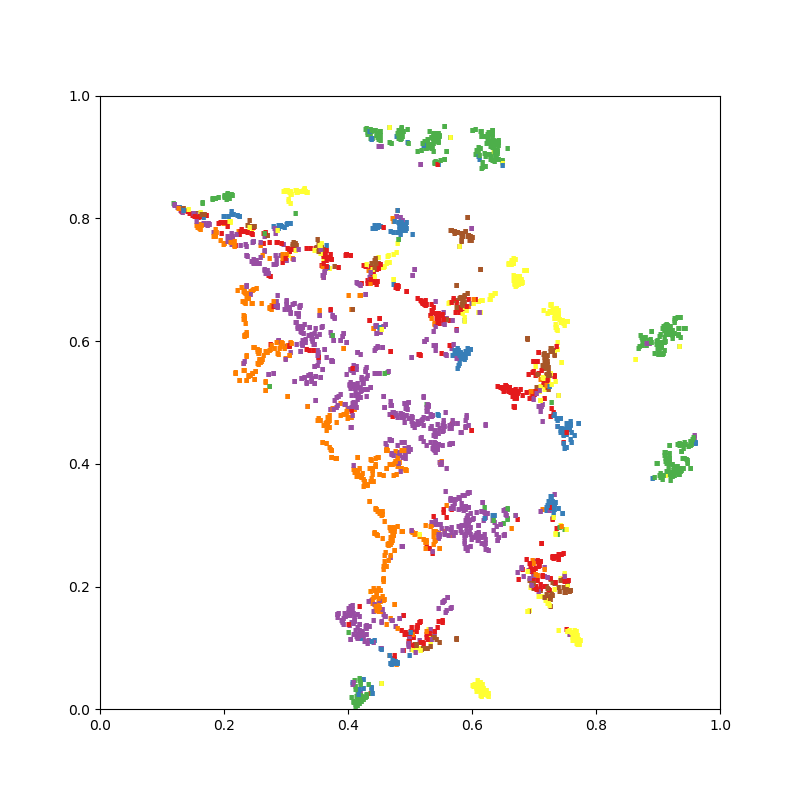}
		\end{minipage}%
	}%
	\subfigure[With topological regularization.]{
		\begin{minipage}[t]{0.49\linewidth}
			\centering
			\includegraphics[width=6.6cm]{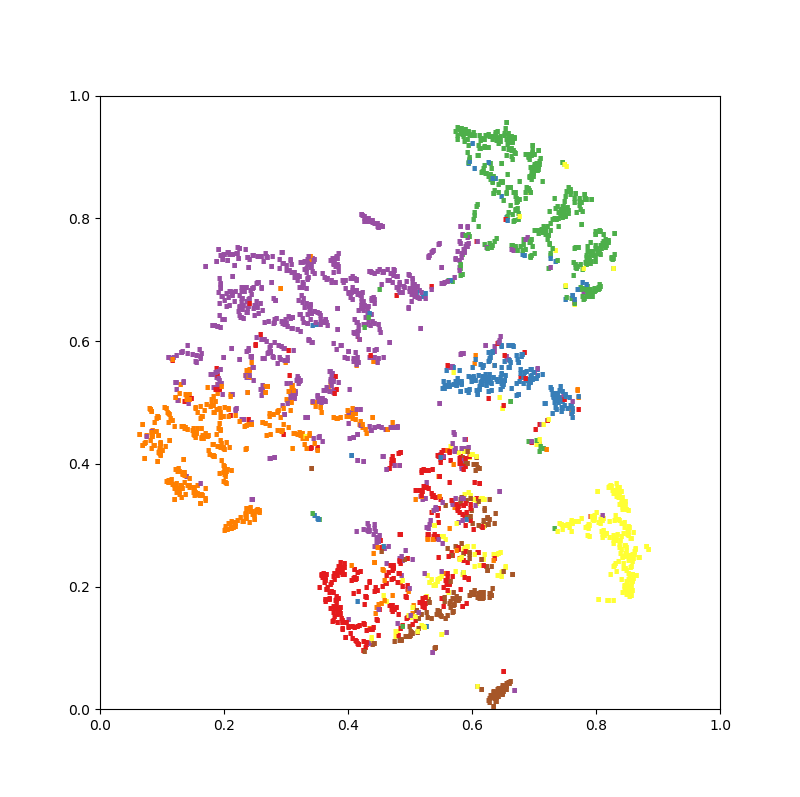}
		\end{minipage}%
	}%
	\caption{Visualization on Cora of GCNII with/without topological regularization.}
	\label{fig:vis}
\end{figure}
We provide evidence that topological regularization prevents overs-mooth through visualization. Often the deep network’s oversmooth phenomenon is obvious, the nodes become indistinguishable after excessive feature propagation. Therefore, we choose the last layer of GCNII on CORA as the output because it is deep enough (64 layers). We use t-SNE \cite{maaten2008visualizing} for dimensionality reduction and visualization of output. In order to minimize the impact of the initial residual $\alpha$ in GCNII on the results and prevent gradient disappearance/explosion, we set $\alpha=0.01$ rather than $0.1$ reported in \cite{chen2020simple}. The overall visualization results are shown in Figure \ref{fig:vis}. The results using the topological regularization term show a more obvious spindle-like distribution, while without topological regularization, the result has a smoother outline, and the inner boundaries of the class are more ambiguous. This shows the anti-over-smooth ability of topological regularization, and also confirms the Theorem \ref{thm4}.

\section{Conclusion}
In this paper, we propose a graph node data augmentation method, which obtains the dense topological features based on the walking method. The topology regularization term is proposed to guide the optimization of the model, and eliminate bad effect of the topological features. Our method can be easily combined with the baseline GNN models and improve the performance of the original model. Theoretical analysis is given to illustrate the effectiveness of topological regularization, which is verified by a large number of experiments. We will continue to study the influence of different topological features on the model performance in the future work.


\bibliography{mybibfile}

\end{document}